\documentclass[9pt,conference]{IEEEtran}
\IEEEoverridecommandlockouts

\usepackage{cite}
\usepackage{amsmath,amssymb,amsfonts}
\usepackage{algorithmic}
\usepackage{graphicx}
\usepackage{textcomp}
\usepackage{xcolor}
\usepackage{fancyhdr}
\usepackage{amssymb, booktabs, mathtools, amsthm}
\usepackage{hyperref}
\usepackage{tcolorbox}
\usepackage{subcaption}

\def\x{{\mathbf x}}
\def\W{{\mathbf W}}

\def\tcW{\widehat{\bs{\mathcal W}}}
\def\calW{\bs{\mathcal W}}

\def\I{{\mathbf I}}

\def\D{{\cal D}}
\def\G{{\cal G}}
\def\V{{\cal V}}
\def\E{{\cal E}}
\def\N{{\cal N}}
\def\y{{\mathbf{y}}}
\def\R{\mathbb{R}}

\def\bTheta{\bs{\Theta}}
\def\bPhi{\bs{\Phi}}
\def\bPsi{\bs{\Psi}}
\def\bvtheta{\bs{\vartheta}}
\def\dbvtheta{\dot{\bs{\vartheta}}}
\def\f{\mathsf{f}}

\newcommand{\btheta}{\boldsymbol{\theta}}
\usepackage{color}
\newcommand{\bs}[1]{\boldsymbol{#1}}

\DeclareMathOperator*{\argmin}{arg\,min}
\mathtoolsset{showonlyrefs}

\newtheorem{prop}{Proposition}

\newtcolorbox{mystyle}[1]{
  colback=white,
  boxrule=0.5mm,        
  left=0mm,             
  right=0mm,            
  top=0mm,              
  bottom=0mm,           
  title={\centering #1},
}

\def\BibTeX{{\rm B\kern-.05em{\sc i\kern-.025em b}\kern-.08em
    T\kern-.1667em\lower.7ex\hbox{E}\kern-.125emX}}
\makeatletter
\newcommand{\linebreakand}{%
  \end{@IEEEauthorhalign}
  \hfill\mbox{}\par
  \mbox{}\hfill\begin{@IEEEauthorhalign}
}
\makeatother

\begin{document}
\title{Peer-to-Peer Learning Dynamics of Wide Neural Networks}

\author{\IEEEauthorblockN{Shreyas Chaudhari\IEEEauthorrefmark{1} \qquad
Srinivasa Pranav\IEEEauthorrefmark{1} \qquad Emile Anand\IEEEauthorrefmark{2} \qquad
Jos\'e M. F. Moura\IEEEauthorrefmark{1}}
\IEEEauthorblockA{\IEEEauthorrefmark{1}Electrical and Computer Engineering,
Carnegie Mellon University\\
\IEEEauthorrefmark{2}School of Computer Science,
Georgia Institute of Technology \\
\{\texttt{shreyasc@andrew.cmu.edu}, \texttt{spranav@cmu.edu}, \texttt{emile@gatech.edu}, \texttt{moura@andrew.cmu.edu} \}}
\thanks{Published at IEEE International Conference on Acoustics, Speech and Signal Processing (ICASSP), 2025. DOI: \href{https://ieeexplore.ieee.org/document/10890126}{10.1109/ICASSP49660.2025.10890126}. Authors partially supported by NSF Graduate Research Fellowships (DGE-1745016, DGE-2140739), NSF Grant CCF-2327905, and ARCS Fellowship.}
}

\maketitle

\begin{abstract}
Peer-to-peer learning is an increasingly popular framework that enables beyond-5G distributed edge devices to collaboratively train deep neural networks in a privacy-preserving manner without the aid of a central server. Neural network training algorithms for emerging environments, e.g., smart cities, have many design considerations that are difficult to tune in deployment settings -- such as neural network architectures and hyperparameters. This presents a critical need for characterizing the training dynamics of distributed optimization algorithms used to train highly nonconvex neural networks in peer-to-peer learning environments. In this work, we provide an explicit characterization of the learning dynamics of wide neural networks trained using popular distributed gradient descent (DGD) algorithms. Our results leverage both recent advancements in neural tangent kernel (NTK) theory and extensive previous work on distributed learning and consensus. We validate our analytical results by accurately predicting the parameter and error dynamics of wide neural networks trained for classification tasks.
\end{abstract}

\section{Introduction}
\label{sec:intro}
The peer-to-peer learning framework is an important step toward realizing smart ecosystems~\cite{kairouz2021advances,6Gedge,HVPoorIoTSurvey}. As distributed devices are increasingly gathering vast amounts of data, advancements in hardware are enabling on-device training of neural networks in a privacy-preserving manner~\cite{Camaroptera, deeplearningonmobiledevices}. In beyond-5G environments, iterative consensus-based algorithms facilitate collaborative learning for smart devices, which train deep neural networks using locally-available training data and exchange model parameters with nearby devices. 

Tuning peer-to-peer learning algorithms for deployment is often resource-intensive, particularly in terms of communication cost. For example, the process of fine-tuning classifiers for domain adaptation tasks~\cite{weiss2016survey} is highly sensitive to hyperparameters and can require several trials, even with fixed, pretrained feature extractors~\cite{labonte2023lastlayerretraininggrouprobustness,kirichenko2023layerretrainingsufficientrobustness}. Simulations, while useful, demand significant computational power and can struggle to accurately model the inherent randomness of wireless networks~\cite{6Gedge}.
Therefore, an analytic characterization of the distributed training process would enable efficient and informed decision-making for neural network architecture design, hyperparameter tuning, and computation and communication optimization.

While many works have examined distributed consensus, estimation, and optimization algorithms in diverse settings~\cite{tsitsiklis1986distributed,consensus+innovations,jakovetic2014fast,anand2024efficientreinforcementlearningglobal}, the inherent nonconvexity of modern neural networks makes it especially challenging to analyze peer-to-peer deep learning~\cite{pranav2023, koloskova2020unified}.
When first-order, gradient-based distributed training algorithms are applied to a linear learning model, the resulting discrete-time dynamics and corresponding continuous-time gradient flows lead to tractable characterizations of the parameter and loss evolution; however, the behavior of nonlinear models scale in complexity~\cite{SoummyaSwensonGradientFlow, MirrorDescentNonlinearHassibiJournal}. To facilitate theoretical analysis for nonlinear neural networks, we leverage Neural Tangent Kernel (NTK) theory. In particular, we use the fact that wide (overparameterized) neural networks behave similarly to models that are linear with respect to their parameters~\cite{jacot2018neural,lee2019wide}. Even for highly performant neural networks, the NTK parameterization implies that as the layers of the neural network become wide, a certain kernel matrix specifying the gradient flow becomes time-invariant and leads to linear dynamics that can be readily analyzed.
 
 \textbf{Contributions}: We consider peer-to-peer learning settings with agents using distributed gradient descent (DGD)~\cite{nedic2009distributed} to train overparameterized (wide) neural networks. We propose a framework based on NTK theory to study the gradient flows -- continuous time generalizations of the discrete gradient descent updates -- of various consensus and diffusion-based forms of DGD. Our results enable the dynamics of the parameters at each agent to be computed analytically. The method we propose applies to a broad range of practical scenarios, including situations where each agent either learns the parameters of a wide neural network or fine-tunes the final layer of a neural network. Finally, we provide experimental simulations where we show strong agreement between the predicted dynamics and the true training dynamics for distributed learning tasks involving neural networks that are nonconvex in their parameters.

\textbf{Notation}: We denote matrices in bold uppercase, e.g., $\mathbf{W}, \boldsymbol{\mathcal{W}}$, and vectors in bold lowercase, e.g., $\mathbf{x}$. We let $\|\mathbf{x}\|_2$ be the $\ell_2$-norm of a vector $\mathbf{x}$. We denote the Kronecker product of matrices $\mathbf{A}\in \R^{m\times n}$ and $\mathbf{B}\in \R^{p\times q}$ by $\mathbf{A}\otimes\mathbf{B} \in \R^{pm\times qn}$. For $\btheta \in \R^p$, $\frac{\partial f}{\partial \btheta} := [\frac{\partial f}{\partial \btheta_1} \dots \frac{\partial f}{\partial \btheta_p}]$ is the Jacobian. Finally, we denote the
$P\times 1$ all-ones vector by $\mathbf{1}_P$ and the $P\times P$ identity matrix by $\mathbf{I}_P$.

\section{Neural Tangent Kernel}
\label{sec:ntk}
We first briefly summarize the Neural Tangent Kernel, which motivates the analytical results presented in subsequent sections. Let $f(\x; \btheta): \R^N \to \R^M$ be a feed-forward neural network of the form
\begin{align}
   \x^{(l+1)} &= \sigma^{(l)} \left( \W^{(l)} \x^{(l)} + \mathbf{b}^{(l)}\right)  \label{eq:ntk_nn}
\end{align}
where $\W^{(l)} \in \R^{n_{l+1} \times n_{l}}, \mathbf{b}^{(l)} \in \R^{n_{l+1}}$, and $\sigma^{(l)}$ respectively denote the weight, bias, and elementwise activation function at layer $l$. The weights and biases are parameterized as 
\begin{align}
    \W^{(l)} &= \frac{s_{W}}{\sqrt{n}_{l}} \mathsf{W}^{(l)},\;\mathbf{b}^{(l)} = s_b \mathsf{b}^{(l)}
\end{align}
where $\mathsf{W}^{(l)}$ and $\mathsf{b}^{(l)}$ are the trainable parameters at each layer. To initialize the trainable parameters, each entry is independently sampled from the standard normal distribution, i.e., $\mathsf{W}_{i,j}^{(l)}, \mathsf{b}_{i}^{(l)} \sim \N(0, 1)$. The parameters $s_{W}$ and $s_b$ are fixed weight and bias standard deviations that are specified at initialization and kept fixed during training. The aforementioned parameterization normalizes both the forward and backward propagation dynamics of a neural network and does not affect the space of parameterized functions \cite{lee2019wide, jacot2018neural}. The empirical neural tangent kernel associated with $f$ is defined as
\begin{align}
    \widehat{\bTheta}(\x, \x') \triangleq \left(\frac{\partial f(\x; \btheta)}{\partial \btheta}\right)\left(\frac{\partial f(\x';\btheta)}{\partial \btheta}\right)^\top
\end{align}
Ref.~\cite{jacot2018neural} proves that if $f$ employs the parameterization described, then the empirical NTK converges, in probability, to a fixed limiting kernel $\bTheta$ as the layer widths tend to infinity. Furthermore, ref.~\cite{jacot2018neural} shows that the empirical NTK asymptotically remains constant during training, thereby implying that wide neural networks are linear with respect to their parameters~\cite{liu2020linearity}. The NTK hence significantly facilitates analysis of neural network training dynamics and has been used to effectively characterize neural network training in various contexts, including generative adversarial networks~\cite{franceschi2022neural} and graph neural networks~\cite{du2019graph}.

\section{System Model}
\label{sec:system_model}

\subsection{Setup}
We consider a peer-to-peer learning setting with $Q$ agents. All agents possess a common neural network architecture, $f(\x; \btheta) : \R^{N} \to \R^{M}$, that is parameterized by $\btheta \in \R^P$. Each agent has its own local objective function $L_{q}(\btheta) : \R^P \to \R$ defined using a local dataset $\D_{q} \triangleq \{(\x_{q,i}, \y_{q,i})\}_{i=1}^D$. For example, the local objective $L_{q}(\btheta)$ can be the local empirical risk function that evaluates the performance of $f(\x; \btheta)$ on a task involving the local dataset $\D_{q}$. The agents cooperate to learn the optimal $\btheta^*$ that minimizes the global objective, which is a weighted average of all local objectives. Formally, the agents collaboratively aim to find $\btheta^*$ satisfying: 
\begin{align}
    \btheta^*\!&=\! \argmin_{\btheta \in \R^P} L(\btheta) \label{eq:global_obj}\\ 
    L(\btheta)\!&=\!\frac{1}{Q}\sum_{q=1}^Q L_{q}(\btheta) \!\triangleq\! \frac{1}{Q}\sum_{q=1}^{Q} \frac{1}{D} \sum_{i=1}^{D}\!\ell\! \left(f\!\left(\x_{q, i}; \btheta\right)\!, \y_{q,i}\right) \label{eq:local_obj}
\end{align}
We assume that the communication among the agents can be modeled as an undirected, static, connected, graph $\G(\V, \E)$ where $\V = \{1,2,\dots, Q\}$ is the node set and $\E \subseteq \V \times \V$ is the edge set. An agent $q$ communicates only with its immediate neighbors $\N_q \subseteq \V$, where $q' \in \N_q$ if $(q',q) \in \E$ or if $q'=q$.\footnote{This simplifies notation and accounts for agents using their own local data.}

\subsection{Distributed Learning Algorithms}
Common algorithms for solving~\eqref{eq:global_obj} include variants of distributed gradient descent (DGD)~\cite{nedic2009distributed} based on consensus and diffusion updates~\cite{sayed2014adaptation, vlaski2023networked, kar2008distributed,cattivelli2009diffusion}. Throughout the paper, we refer to the consensus-based algorithm as DGD.
In each iteration of DGD, every agent first averages the parameter estimates of its neighbors. Each agent then updates this averaged value using the gradient of its local objective with respect to its parameter estimate before averaging. Hence, DGD iteratively refines the local estimate of $\btheta^*$ by updates of the following form:
\begin{align}
    \btheta_{q, k+1} &= \sum_{r \in \N_q} w_{qr}\btheta_{r, k} - \eta\left(\frac{\partial L_{q}}{\partial \btheta_{q, k}}\right)^\top\label{eq:dgd_update},\; k \geq 0
\end{align}
where $\btheta_{q,0} \in \R^P$ is the initial estimate of $\btheta^*$ at agent $q$ and $w_{qr}$ denotes the weight that agent $q$ assigns to the parameter estimate of agent $r$. The DGD update in~\eqref{eq:dgd_update} is closely related to the Adapt-Then-Combine (ATC) diffusion update \cite{chen2012diffusion}:
\begin{align}
 \btheta_{q,{k+1}} &= \sum_{r \in \mathcal{N}_q} w_{qr} \left[\btheta_{r,{k}} -\eta \left(\frac{\partial L_r}{\partial \btheta_{r,k}}\right)^\top\right] \label{eq:atc_update}
\end{align}
While each agent in DGD first averages parameter estimates of its neighbors, the agents in ATC diffusion first perform a gradient update step on their local estimates before averaging the estimates of their neighbors.
Reversing the order of the ATC update to average neighboring estimates before applying a gradient update based on the averaged estimate yields the combine-then-adapt (CTA) diffusion update \cite{chen2012diffusion}: 
\begin{align}
    \bs{\psi}_{q, k} &\triangleq \sum_{r \in \mathcal{N}_q} w_{qr} \btheta_{r,k}\\
    \btheta_{q, k+1} &= \bs{\psi}_{q, k} - \eta \left(\frac{\partial L_q}{\partial \bs{\psi}_{q, k}}\right) \label{eq:ctc_update}
\end{align}
Let $\W \in \R^{Q\times Q}$ be the weighted adjacency matrix (with self-loops) that collects the mixing coefficients $w_{qr}$.
Let $\bvtheta_{k} \triangleq [\btheta_{1,k}^\top, \btheta_{2,k}^\top, \cdots, \btheta_{Q,k}^\top]^\top \in \R^{QP}$ stack the parameter estimates across all agents at iteration $k$. We can compactly express the DGD, ATC diffusion, and CTA diffusion updates in ~\eqref{eq:dgd_update},~\eqref{eq:atc_update},~\eqref{eq:ctc_update} as:
\begin{mystyle}{Gradient Updates}
    \begin{subequations}
    \begin{align}
    \calW &\triangleq \W \otimes \I_P\\
    \text{DGD:}\; \bvtheta_{k+1} &= \calW \bvtheta_{k} - \eta\left(\frac{\partial L}{\partial \bvtheta_k}\right)^\top \label{eq:vectorized_dgd}\\
    \text{ATC:}\; \bvtheta_{k+1} &= \calW \left(\bvtheta_{k} -\eta \left(\frac{\partial L}{\partial \bvtheta_{t}}\right)^\top\right) \label{eq:vectorized_atc}\\
    \text{CTA:}\; \bvtheta_{k+1} &=\calW\bvtheta_{k} -\eta \left(\frac{\partial L}{\partial \calW \bvtheta_k}\right)^\top \label{eq:vectorized_cta} 
\end{align}
\end{subequations}
\end{mystyle}

The updates in~\eqref{eq:vectorized_dgd},~\eqref{eq:vectorized_atc}, and~\eqref{eq:vectorized_cta} above are first order difference equations. To derive closed-form solutions for the dynamics associated with neural networks trained using these algorithms, it is convenient to analyze their continuous-time counterparts, namely their gradient flows: 
\begin{mystyle}{Gradient Flows}
    \begin{subequations}
    \begin{align}
    \tcW &\triangleq (\W - \I_{Q}) \otimes \I_P \\
    \text{DGD:}\; \dbvtheta_{t} &=  \tcW \bvtheta_{t} -\eta  \left(\frac{\partial L}{\partial \bvtheta_{t}}\right)^\top \label{eq:dgd_gradflow}\\ 
    \text{ATC:}\; \dbvtheta_{t} &= \tcW \bvtheta_{t} -\eta \calW \left(\frac{\partial L}{\partial \bvtheta_{t}}\right)^\top \label{eq:atc_gradflow}\\
    \text{CTA:}\; \dbvtheta_{t} &=  \tcW\bvtheta_t -\eta \left(\frac{\partial L}{\partial \calW \bvtheta_t}\right)^\top \label{eq:cta_gradflow}
\end{align}
\end{subequations}
\end{mystyle}
The standard gradient updates in~\eqref{eq:vectorized_dgd},~\eqref{eq:vectorized_atc}, and~\eqref{eq:vectorized_cta} can be viewed as Euler discretizations of the continuous time gradient flows in~\eqref{eq:dgd_gradflow},~\eqref{eq:atc_gradflow}, and~\eqref{eq:cta_gradflow} respectively.

\section{Peer-to-Peer Learning Dynamics}
\label{sec:dynamics}
We consider solving the gradient flow equations in~\eqref{eq:dgd_gradflow},~\eqref{eq:atc_gradflow}, and~\eqref{eq:cta_gradflow} to determine the evolution of the parameter estimates $\btheta_{q,t}$ at each agent. Let $f_{q,t}(\x) \triangleq f(\x; \btheta_{q,t})$ denote the neural network $f : \R^N \to \R^M$ evaluated on input $\x$ using the parameter estimate $\btheta_{q,t}$ at agent $q$ and time $t$. We collect the output of $f$ evaluated on all local datasets using the local parameters at time $t$ as follows:
\begin{align}
    f_t(\D_{q}) \in \R^{MD} &\triangleq [f_{q,t}(\x_{q,1})^\top,\cdots,f_{q,t}(\x_{q,D})^\top]^\top \label{eq:local_fvector}\\
    \f_t \in \R^{QMD} &\triangleq [f_t(\D_1)^\top,\cdots,f_t(\D_Q)^\top]^\top \label{eq:global_fvector}
\end{align}
In~\eqref{eq:local_fvector},  $f_t(\D_{q})$ collects the neural network outputs on the local dataset $\D_{q}$ using the local parameters $\btheta_{q,t}$. The vector $\f_t$ in~\eqref{eq:global_fvector} collects all of the vectors $f_t(\D_{q})$ across the $Q$ agents. Solving the gradient flow equations in~\eqref{eq:dgd_gradflow},~\eqref{eq:atc_gradflow}, and~\eqref{eq:cta_gradflow} is challenging in general, as $\left(\frac{\partial L}{\partial \bvtheta_t}\right)^\top = \left(\frac{\partial \f_t}{\partial \bvtheta_t}\right)^\top \left(\frac{\partial L}{\partial \f_t}\right)^\top $ is highly nonlinear in $\btheta_t$ for standard neural networks $f$ due to the term $\left(\frac{\partial \f_t}{\partial \bvtheta_t}\right)^\top$. However, according to NTK theory~\cite{lee2019wide, jacot2018neural}, the term $\left(\frac{\partial \f_t}{\partial \bvtheta_t}\right)^\top$ is linear with respect to $\bvtheta$ in the limit as the neural network width tends to infinity. Therefore, for sufficiently wide neural networks $f$ at each agent, we consider linearizing $\f_{t}$ about the initial parameters $\bvtheta_0$:
\begin{align}
    \overline{\f}_t &\triangleq \f_0 + \frac{\partial \f}{\partial \bvtheta_{0}}\left({\bvtheta}_{t} - \bvtheta_{0}\right) \label{eq:linearization}
\end{align}
For mean squared error (MSE) loss, the dynamics of $\dot{\overline{\f}}_t, \dbvtheta_t$ can be solved in closed form. Let $\ell(\x, \y) = \frac{1}{2}\|\x - \y\|_2^2$ so that the global objective $L(\btheta)$ in~\eqref{eq:local_obj} becomes: 
\begin{align}
      L(\btheta) &= \frac{1}{Q}\sum_{q=1}^{Q} \frac{1}{2D} \sum_{i=1}^{D} \left\|f\left(\x_{q, i}; \btheta\right) - \y_{q,i}\right\|_2^2
\end{align}
Substituting the the linearization $\overline{\f}_t$ into the gradient flow dynamics in \eqref{eq:dgd_gradflow}, \eqref{eq:atc_gradflow}, and \eqref{eq:cta_gradflow} yields the following linearized gradient flows:

\begin{mystyle}{Linearized Gradient Flows for MSE Loss}
    \scriptsize 
    \begin{subequations}
    \begin{align}
    \bPsi_0 &\triangleq \left(\frac{\partial \f}{\partial \bvtheta_{0}}\right)^\top \left(\frac{\partial \f}{\partial \bvtheta_{0}}\right)\\
    \text{DGD:}\;\dbvtheta_{t}\!&=\!\left[\tcW - \frac{\eta}{DQ} \bPsi_0\right] \bvtheta_t \!- \!\frac{\eta}{DQ} \left[ \left(\frac{\partial {\f}}{\partial \bvtheta_0}\right)^\top (\f_0 - \mathsf{y}) - \bPsi_0 \bvtheta_0 \right] \label{eq:dgd_lin_gradflow}\\
    \text{ATC:}\;\dbvtheta_{t}\!&=\!\left[\tcW\!- \!\frac{\eta}{DQ} \calW \bPsi_0\right]\bvtheta_t\!  -\!  \frac{\eta}{DQ} \calW \left[\left(\frac{\partial {\f}}{\partial \bvtheta_0}\right)^\top (\f_0 \!-\! \mathsf{y})\!-\!\bPsi_0 \bvtheta_0 \right] \label{eq:atc_lin_gradflow} \\
    \text{CTA:}\;\dbvtheta_{t}\!&=\!\left[\tcW- \frac{\eta}{DQ} \bPsi_0 \calW \right]\bvtheta_t  - \frac{\eta}{DQ} \left[ \left(\frac{\partial {\f}}{\partial \bvtheta_0}\right)^\top (\f_0 - \mathsf{y}) - \bPsi_0 \bvtheta_0 \right] \label{eq:cta_lin_gradflow}
\end{align}
\end{subequations}
\end{mystyle}
The linearized gradient flows for MSE loss can be solved using the general solution:
\begin{align}
    \bPhi(t, t_0) &= \exp \left(\int_{t_0}^t \mathbf{A}(s) ds\right) = \exp\left(\mathbf{A}(t - t_0)\right) \label{eq:state_transition}\\ 
    \bvtheta_t &= \bPhi(t,0)\bvtheta_0 + \left(\int_0^t \bPhi (t, \tau) d\tau\right) \mathbf{u} \label{eq:general_solution}
\end{align}
where the state transition matrix $\bPhi(t, t_0)$ has the form~\eqref{eq:state_transition} since the state matrices $\mathbf{A} = \tcW - \eta \bPhi_0$, $\mathbf{A} = \tcW - \eta \calW \bPhi_0$, and $\mathbf{A} = \tcW - \eta \bPhi_0 \calW$ in~\eqref{eq:dgd_lin_gradflow},~\eqref{eq:atc_lin_gradflow}, and ~\eqref{eq:cta_lin_gradflow}, respectively, are all time-invariant. In~\eqref{eq:general_solution}, $\mathbf{u}$ corresponds to the constant, rightmost terms in~\eqref{eq:dgd_lin_gradflow},~\eqref{eq:atc_lin_gradflow}, and ~\eqref{eq:cta_lin_gradflow}.  We note that the gradient flow equations with the linearized neural network~\eqref{eq:linearization} can also be accurately solved for other loss functions, e.g., cross entropy, by using numerical integration \cite{lee2019wide}. 
We next analyze the bounded input bounded output (BIBO) stability of the dynamics associated with the aforementioned algorithms. For the sake of brevity, we include the DGD stability proof. 

\begin{prop}[DGD Gradient Flow Stability]
    Let $\W$ be doubly stochastic. Then the gradient flow of $\bvtheta_t$ in~\eqref{eq:dgd_lin_gradflow} is BIBO stable if the system is minimal.
\end{prop}
\begin{proof}
    $\W$ is doubly stochastic, so $\tcW$ has maximum eigenvalue $\lambda = 0$ with multiplicity $P$ and corresponding eigenvectors $\bs{1}_Q \otimes \I_P$. Since $\tcW$ is negative semidefinite, $\tcW - \eta \bPsi_0$ is negative semidefinite. Suppose $\tcW - \eta \bPsi_0$ has a zero eigenvalue with corresponding nonzero eigenvector $\x \in \R^{PQ}$. Then there exists nonzero $\x \in \R^{PQ}$ such that $\x^\top (\tcW - \eta \bPsi_0)\x = 0$, i.e., $\x$ simultaneously lies in the nullspaces of $\tcW$ and $\bPsi_0$. The columns of $\bs{1}_Q \otimes \I_P$ form a basis for the null space $\mathcal{N}(\tcW)$. $\mathcal{N}(\tcW) \cap \mathcal{N}(\bPsi_0)$ is nontrivial only if a column of $\bPsi_0 (\bs{1}_Q \otimes \I_P)$ is zero. Since $\bPsi_0$ is block diagonal,  $\bPsi_0 (\bs{1}_Q \otimes \I_P) = \begin{bmatrix}\frac{\partial f_t(\D_1)}{\partial \theta_{1,0}}^\top, \dots, \frac{\partial f_t(\D_1)}{\partial \theta_{Q,0}}^\top \end{bmatrix}^\top$. If a column of the RHS is zero, there exists a parameter that does not influence the neural network's output. This contradicts minimality of the system.
\end{proof}

\section{Experiments}
\label{sec:experiments}
In this section, we validate the analytical expressions derived in Section~\ref{sec:dynamics} for the training dynamics corresponding to DGD with consensus updates. For all experiments, we consider peer-to-peer networks with synchronous, bidirectional, and noiseless device-to-device (D2D) communication channels. From a global perspective, we model the network as an undirected, flat, and connected communication graph with devices as vertices and D2D links as edges.
We employ Metropolis-Hastings mixing weights that produce symmetric doubly stochastic mixing matrices~\cite{gossip-metropolis-hastings-mixing}. At the start of training, we identically initialize the models at all agents using max norm synchronization~\cite{pranav2024}.

\subsection{Affine Models}
We first demonstrate that the linearized gradient flow in~\eqref{eq:dgd_lin_gradflow} accurately describes the training dynamics of DGD by considering an affine form for $f$ in~\eqref{eq:local_obj}.
We consider a setting with $Q=8$ agents. Each agent has a local dataset of $D=200$ that are independently and identically sampled from the subset of the MNIST dataset \cite{lecun2010mnist} containing handwritten digits of zeros and ones. The agents collaboratively train an affine model to classify images to match the labels $y_{q,i}$, which are either $0$ or $1$, using mean squared error (MSE) loss. While, in practice, classification is usually performed using cross-entropy loss, here we treat the classification task as regression, as in~\cite{lee2019wide}, to facilitate solving the corresponding gradient flow equation. The global optimization objective becomes
\begin{align}
    L(\mathbf{w}, b) &= \frac{1}{Q} \sum_{q=1}^Q \frac{1}{2D} \sum_{i=1}^D (\mathbf{w}^\top \x_{q,i} + b - y_{q,i})^2 \label{eq:linear_exp}
\end{align}

First, we simulate the agents optimizing the global objective. The local parameter estimates at each agent are updated using the discrete gradient updates in~\eqref{eq:dgd_update} for 200 iterations and with step size $\eta = 10^{-4}$. In this experiment, the agents communicate their parameters along a fixed, complete graph. We compare these observed dynamics to the those predicted by solving the linearized gradient flow equation for DGD in~\eqref{eq:dgd_lin_gradflow}. The MSE training loss dynamics for each agent are shown in Figure~\ref{fig:linear}, with each color representing the dynamics corresponding to a specific agent. The dashed lines indicate the predicted loss values obtained by solving the linearized gradient flow in~\eqref{eq:dgd_lin_gradflow}. The parameters obtained by solving the gradient flow are then substituted into the affine classifier to yield the predicted MSE loss. The true loss values observed while training the classifier across the agents are given by the solid lines.
Since the classifier model is affine, the solution to~\eqref{eq:dgd_lin_gradflow} exactly corresponds to the observed loss dynamics.
\begin{figure}[htpb]
    \centering
    \includegraphics[width=\linewidth]{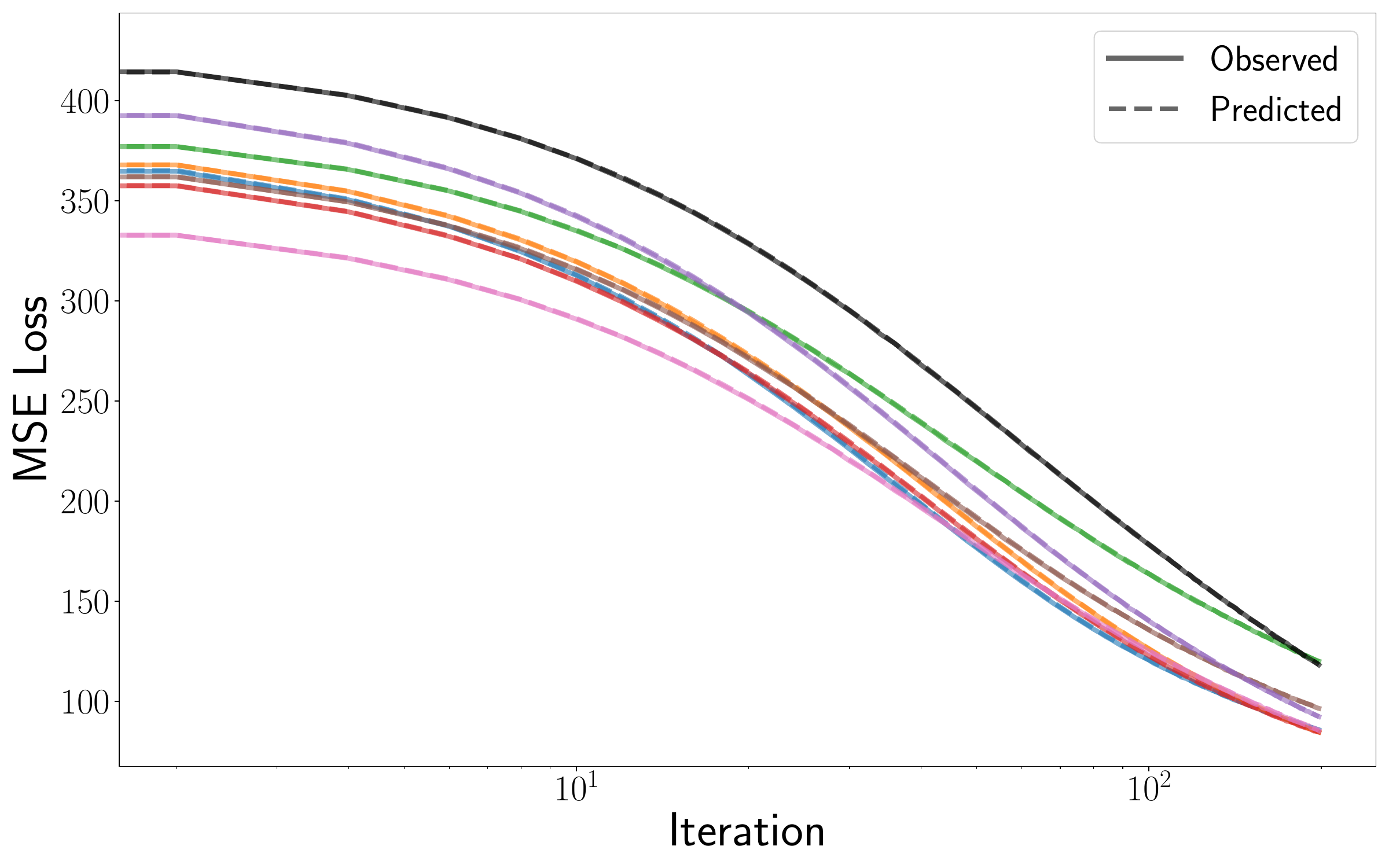}
        \caption{Loss dynamics for each agent in a complete network communication graph solving~\eqref{eq:linear_exp} with DGD (affine classifier $f$).}
    \label{fig:linear}
\end{figure}

\subsection{Neural Networks}
Similar to the previous experiment, we now examine a scenario with $Q=8$ agents where each agent has a neural network classifier. The neural networks are of the form described in~\eqref{eq:ntk_nn}, with a single hidden layer of dimension 256 and sigmoid activation function. The weights and biases are all initialized with standard deviations $s_W= 1$ and $s_b = 0.1$ respectively. Each agent is trained to classify the half moons dataset~\cite{rozza2014novel} using $D=200$ independently and identically distributed samples. The neural network parameter estimates at each agent are updated for 200 iterations using the DGD update in~\eqref{eq:dgd_update}, with step size $\eta = 10^{-4}$. We solve the DGD gradient flow equation in~\eqref{eq:dgd_lin_gradflow} to predict the loss dynamics of the neural networks at each agent during training. The predicted and observed dynamics corresponding to cycle, star, and complete graph communication networks are shown in Figure~\ref{fig:cycle},~\ref{fig:star}, and~\ref{fig:complete} respectively. We observe that although the neural networks are nonconvex in their parameters, our analytical expressions accurately track the learning dynamics of the parameters and therefore accurately characterize the error decay during training. 

\section{Conclusion}
In this paper, we provided an explicit characterization of the learning dynamics of wide neural networks. We considered the continuous time generalizations of consensus and diffusion-based distributed learning algorithms, and derived the gradient flow dynamics of neural network parameters trained with these methods. We analyzed the stability of DGD and empirically demonstrated that our analytical expressions accurately model the training dynamics of such algorithms, even when the models considered are highly nonconvex with respect to their parameters. Our work provides insight for practitioners and aids in the tuning and deployment of effective peer-to-peer learning algorithms. Our proposed framework offers a foundation for further research analyzing nonconvex optimization algorithms in wireless peer-to-peer environments.

\begin{figure}[htpb]
    \centering
    \begin{subfigure}{0.49\textwidth}
     \centering
    \includegraphics[width=1\linewidth]{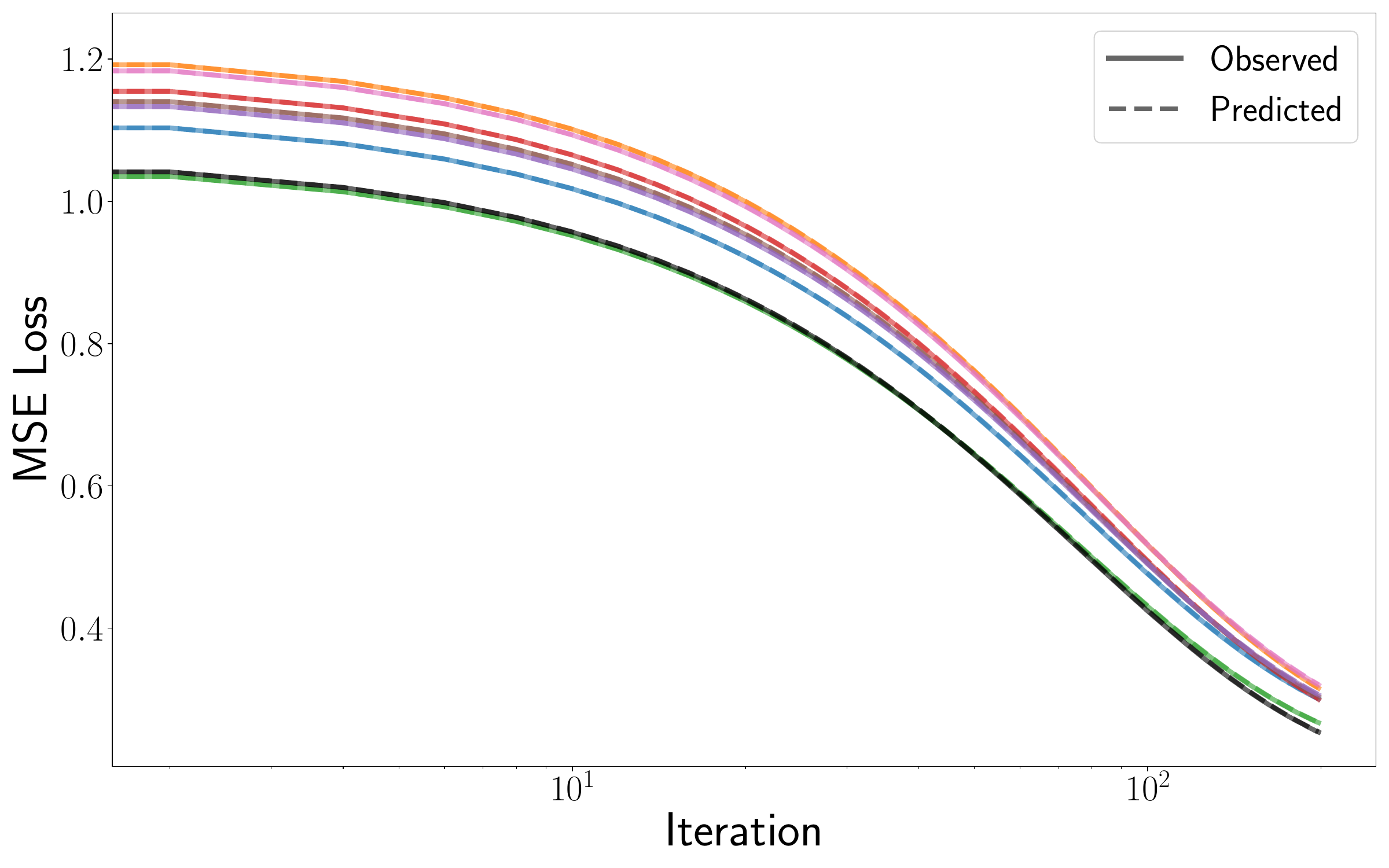}
        \caption{}
    \label{fig:cycle}
    \end{subfigure}
    \hfill
    \begin{subfigure}{0.49\textwidth}
    \centering
    \includegraphics[width=1\linewidth]{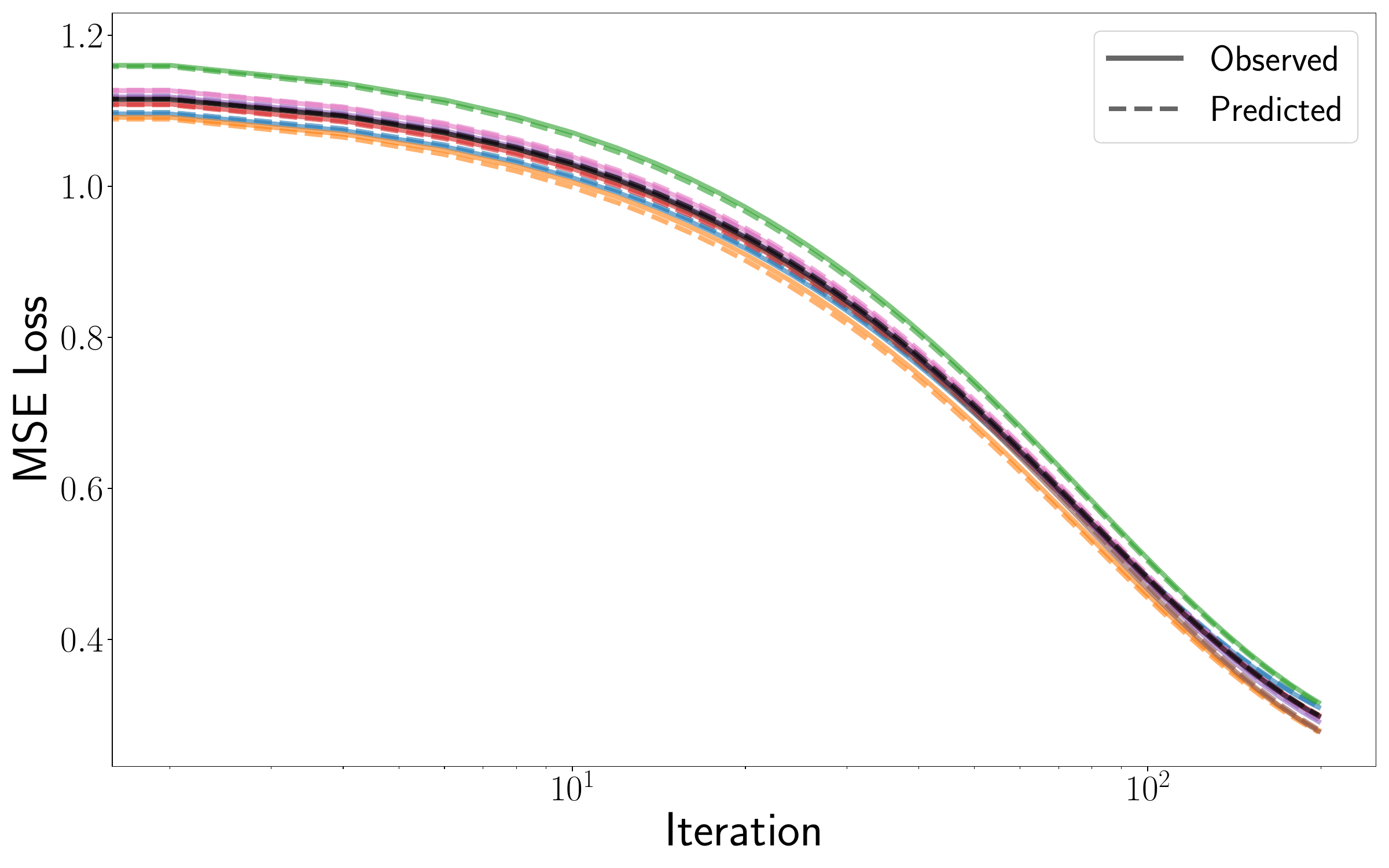}
        \caption{}
    \label{fig:star}
    \end{subfigure}
    \hfill
    \begin{subfigure}{0.49\textwidth}
    \centering
    \includegraphics[width=1\linewidth]{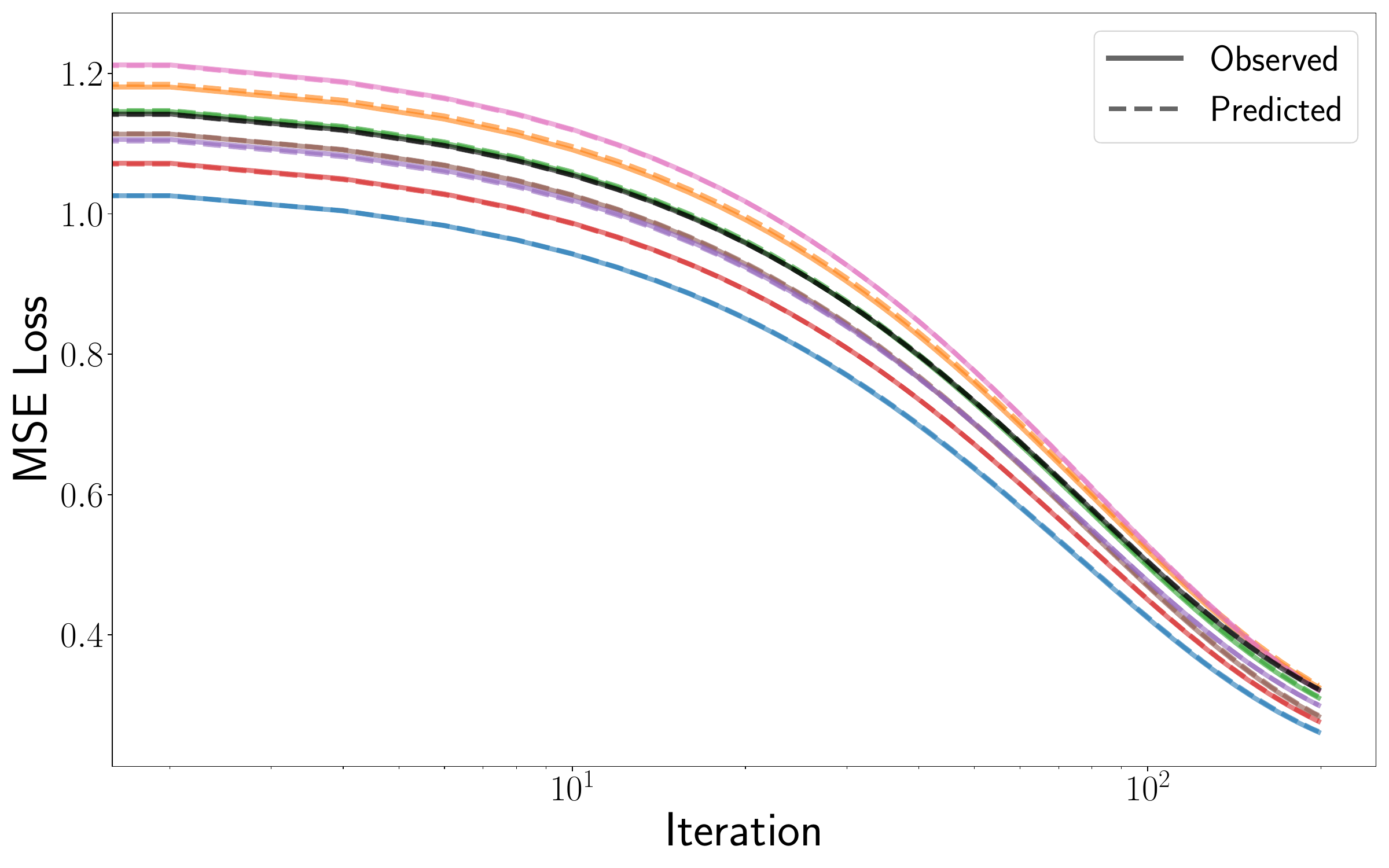}
        \caption{}
    \label{fig:complete}
    \end{subfigure}
    \caption{Loss dynamics for each agent for solving~\eqref{eq:linear_exp} with DGD and neural network classifier $f$ over (a) cycle graph (b) star graph and (c) complete graph communication networks.}
    \label{fig:overall}
\end{figure}

\newpage
\bibliographystyle{IEEEtran}
\bibliography{refs}

\end{document}